\documentclass[10pt,twocolumn,letterpaper]{article}

\usepackage{cvpr}
\usepackage{times}
\usepackage{epsfig}
\usepackage{graphicx}
\usepackage{amsmath}
\usepackage{amssymb}

\usepackage{amsthm}
\newtheorem{lemma}{Lemma}
\newtheorem{prop}{Proposition}
\newtheorem{mydef}{Definition}
\newtheorem{remark}{Remark}

\usepackage{booktabs} %

\usepackage{nopageno}

\usepackage{dblfloatfix}

\usepackage[tight,normalsize,sf,SF]{subfigure}

\usepackage{soul}
\newcommand{\add}[1] {{#1}} %

\newtheorem{problem}{Problem}

\DeclareMathOperator{\im}{im}
\DeclareMathOperator{\diag}{diag}

\newcommand{\R}{\mathbb{R}}

\newcommand{\zerovec}{\mathbf{0}}
\newcommand{\matI}{\mathbf{I}}
\newcommand{\matX}{\mathbf{X}}

\newcommand{\matQ}{\mathbf{Q}}

\newcommand{\matD}{\mathbf{D}}
\newcommand{\matU}{\mathbf{U}}

\newcommand{\matV}{\mathbf{V}}
\newcommand{\matW}{\mathbf{W}}

\newcommand{\matT}{\mathbf{T}}
\newcommand{\matZ}{\mathbf{Z}}
\newcommand{\matN}{\mathbf{N}}

\newcommand{\matSigma}{\mathbf{\Sigma}}

\newcommand{\matA}{\mathbf{A}}
\newcommand{\vect}{\mathbf{t}}

\newcommand{\vecz}{\mathbf{z}}
\newcommand{\vecu}{\mathbf{u}}
\newcommand{\vecx}{\mathbf{x}}

\newcommand{\equivalent}{\Leftrightarrow~}

\newcommand{\unif}{\mathcal{U}}

\usepackage[breaklinks=true,bookmarks=false]{hyperref}

\cvprfinalcopy %

\def\cvprPaperID{72} %

\ifcvprfinal\fi

\begin{document}

\title{A Solution for Multi-Alignment by Transformation Synchronisation}

\author{
Florian Bernard\textsuperscript{1,2,3} 
  \and
Johan Thunberg\textsuperscript{2}  
  \and
Peter Gemmar\textsuperscript{3} 
  \and
Frank Hertel\textsuperscript{1} 
  \and
Andreas Husch\textsuperscript{1,2,3} 
  \and
Jorge Goncalves\textsuperscript{2}
\and
\\
\textsuperscript{1}Centre Hospitalier de Luxembourg, Luxembourg \\
{\tt\small \{bernard.florian,hertel.frank,husch.andreas\}@chl.lu}\\
\textsuperscript{2}Luxembourg Centre for Systems Biomedicine, University of Luxembourg, Luxembourg\\
{\tt\small \{johan.thunberg,jorge.goncalves\}@uni.lu}\\
\textsuperscript{3}Trier University of Applied Sciences, Germany\\
{\tt\small p.gemmar@hochschule-trier.de}
}

\maketitle
\begin{abstract}
The alignment of a set of objects by means of transformations plays an important role in computer vision. Whilst the case for only two objects can be solved globally, when multiple objects are considered usually iterative methods are used. In practice the iterative methods perform well if the relative transformations between any pair of objects are free of noise. However, if only noisy relative transformations are available (\eg due to missing data or wrong correspondences) the iterative methods may fail.

Based on the observation that the underlying noise-free transformations can be retrieved from the null space of a matrix that can directly be obtained from pairwise alignments, this paper presents a novel method for the synchronisation of pairwise transformations such that they are transitively consistent. 

Simulations demonstrate that for noisy transformations, a large proportion of missing data and even for wrong correspondence assignments the method delivers encouraging results.
\end{abstract}

\footnotetext{\copyright ~ 2015 IEEE. Personal use of this material is permitted. Permission from IEEE must be obtained for all other uses, in any current or future media, including reprinting/republishing this material for advertising or promotional purposes, creating new collective works, for resale or redistribution to servers or lists, or reuse of any copyrighted component of this work in other works.}

\section{Introduction}
The alignment of a set of objects by means of transformations plays an important role in the field of computer vision and recognition. For instance, for the creation of statistical shape models (SSMs) \cite{Cootes:1992uw} training shapes are initially aligned for removing pose differences in order to only model shape variability. 

The most common way of shape representation is by encoding each shape as a point-cloud. In order to be able to process a set of shapes it is necessary that correspondences between all shapes are established. Whilst there is a vast amount of research in the field of shape correspondences (for an overview see \cite{Heimann:2009kv,VanKaick:2011uq}), in this paper we focus on the alignment of shapes and we assume that correspondences have already been established. 

The alignment of two objects by removing location, scale and rotation is known as \emph{Absolute Orientation Problem (AOP)} \cite{Horn:1988bq} or \emph{Procrustes Analysis} \cite{Gower:2004uu}. For the AOP there are various closed-form solutions, among them methods based on singular value decomposition (SVD) \cite{Arun:1987uu,Schonemann:1966ch}; based on eigenvalue decomposition \cite{Horn:1988bq}; based on unit quaternions \cite{Horn:1987hf} or based on dual quaternions \cite{Walker:1991kt}. A comparison of these methods \cite{Eggert:1997gf} has revealed that the accuracy and the robustness of all methods are comparable. 

The alignment of more than two objects is known as \emph{Generalised Procrustes Analysis (GPA)}. Whilst a computationally expensive global solution for GPA in two and three dimensions has been presented in \cite{Pizarro:2011ta}, the most common way for solving the GPA is to align the objects with a reference object. However, fixing any of the objects as reference induces a bias. An unbiased alternative is to align all objects with the adaptive mean object as reference. An iterative algorithm then alternatingly updates the reference object and estimates the transformations aligning the objects. The iterative nature of these methods constitutes a problem if the relative transformation between any pair of objects is noisy. This is for example the case if data is missing, correspondences are wrong or if the transformations are observed by independent sensors (\eg non-communicating robots observe each other). Noisy relative transformations can be characterised by transitive inconsistency, \ie transforming $A$ to $B$ and $B$ to $C$ might lead to a different result than transforming $A$ directly to $C$.

This paper presents a novel method for synchronising the set of all pairwise transformations in such a way that they globally exhibit transitive consistency. Experiments demonstrate the effectiveness of this method in denoising noisy pairwise transformations. Furthermore, using this novel method the GPA is solved in an unbiased manner in closed-form, \ie non-iterative. Transformation synchronisation is applied to solve the GPA with missing data as well as with wrong correspondence assignments and results in superior performance compared to existing methods.

Our main contribution is a generalisation of the techniques presented by Singer \etal \cite{Chaudhury:2013un,Hadani:2011tw,Hadani:2011hb,Singer:2011ba}, who have introduced a method for minimising global self-consistency errors between pairwise orthogonal transformations based on eigenvalue decomposition and semidefinite programming. With permutation transformations being a subset of orthogonal transformations, in \cite{Pachauri:2013wx} the authors demonstrate that the method by Singer \etal is also able to effectively synchronise permutation transformations for globally consistent matchings. 

In our case, rather than considering the special case of orthogonal matrices, we present a synchronisation method for invertible linear transformations. Furthermore, it is demonstrated how this method can be applied for the synchronisation of similarity, euclidean and rigid transformations, which are of special interest for the groupwise alignment of shapes.

\add{Whilst the proposed synchronisation method is applicable in many other fields where noisy pairwise transformations are to be denoised} (\eg groupwise image registration or multi-view registration), in this paper GPA is used as illustrating example.

\section{Methods}
For the presentation of our novel transformation synchronisation method the notation and some foundations are introduced first. Subsequently, a formulation for the case of perfect information is given. Motivated by these elaborations, a straightforward extension to handle noisy pairwise transformations is presented.
Finally, various types of transformations are discussed.

\subsection{Notation and Foundations}
$\matX_i, \matX_j \in \R^{n \times d}$ are matrices representing point-clouds with $n$ points in $d$ dimensions where in the following all $\matX_i$ are simply referred to as point-clouds. Let $\matI$ be the identity matrix and $\zerovec$ be the vector containing only zeros, both having appropriate dimensions according to their context. The Frobenius norm is denoted by $\|\cdot\|_F$.
Let $\matT_{ij} \in \R^{d \times d}$ be an invertible transformation matrix aligning point-cloud $\matX_i$ with $\matX_j$ (for all $i,j=1,\ldots,k$), where $\matT_{ij} = \matT_{ji}^{-1}$.
 Furthermore, $\mathcal{T} = \{\matT_{ij}\}_{i=1,j=1}^{k}$ is the set of all $k^2$ pairwise transformations. 

\noindent A desirable property of the set of transformations $\mathcal{T}$ is that it complies with the following transitive consistency condition:
\begin{mydef}\label{transitivity}
The set of relative transformations $\mathcal{T}$ is said to be transitively consistent if 
$$ \matT_{ij} \matT_{jl} = \matT_{il} \quad \text{for all} \quad i,j,l= 1,\ldots,k \,. $$
\end{mydef}
Definition \ref{transitivity} states that the transformation from $i$ to $j$ followed by the transformation from $j$ to $l$ must be the same as directly transforming from $i$ to $l$. 

\begin{lemma}\label{lem1}
The set of relative transformations $\mathcal{T}$ is transitively 
consistent if and only if there is a set of invertible transformations
$\{\bar{\matT}_i\}_{i = 1}^k$ such that 
$$\matT_{ij} = \bar{\matT}_i\bar{\matT}_j^{-1} \quad \text{for all} \quad ~ i,j =  1,\ldots,k \,.$$
\end{lemma}

\begin{proof} For the sake of completeness a proof is provided.

\noindent``$\Leftarrow$'':
Transitive consistency of $\mathcal{T}$ follows directly from the definition $\matT_{ij} = \bar{\matT}_i\bar{\matT}_j^{-1}$, since for all $i,j,l= 1,\ldots,k$ it holds that
\begin{align}
\matT_{il} & = \bar{\matT}_{i}\bar{\matT}_l^{-1} = \bar{\matT}_{i} \matI \bar{\matT}_l^{-1} \\
 &= \bar{\matT}_{i} (\bar{\matT}_{j}^{-1} \bar{\matT}_j) \bar{\matT}_l^{-1} \\
 & =  (\bar{\matT}_{i} \bar{\matT}_{j}^{-1}) (\bar{\matT}_j \bar{\matT}_l^{-1})  \\
 & = \matT_{ij} \matT_{jl} \, .
\end{align}

\noindent``$\Rightarrow$'':
In the 
rest of the proof we direct  
our attention towards the necessity of the existence of the $\bar{\matT}_i$ 
transformations. 

First of all, if the transformations in $\mathcal{T}$ are 
transitively consistent 
\begin{equation}\label{eq:identity} 
\matT_{ii} = \matI \quad \text{for all} \quad ~ i =  1,\ldots,k \,.
\end{equation}
This follows
by the fact that the $\matT_{ii}$ needs to be invertible while 
satisfying, by Definition \ref{transitivity}, that $\matT_{ii}\matT_{ii} = \matT_{ii}$.

Let $\bar{\matT}'_i = \matT_{i1}$ for all $i$. Now we show that 
$\bar{\matT}'_i$ are such $\bar{\matT}_i$ matrices we seek. 
Since, by using \eqref{eq:identity}, $\bar{\matT}'_1 = \matI$, we can write $\matT_{i1} = \bar{\matT}'_{i} \matI  =  \bar{\matT}'_{i} (\bar{\matT}'_{1})^{-1}$.

Now for any $\matT_{ij}$, we can use that
\begin{align}
\matT_{1i}\matT_{ij} = \matT_{1j} \, .
\end{align}
Thus, 
\begin{align}
\matT_{ij} = \matT_{1i}^{-1}\matT_{1j} =  \matT_{i1}\matT_{j1}^{-1}  = \bar{\matT}'_i (\bar{\matT}'_j)^{-1}  \, . %
\end{align}
\end{proof}

\subsection{Perfect Information}
Due to Lemma~\ref{lem1}, there is a
reference coordinate frame, denoted by $\star$,  
from which there are $\matT_{i \star}$ transformations such that $\matT_{ij} = \matT_{i \star} \matT_{\star j}$ for all $i,j$. \add{Note that the reference coordinate frame is merely used as a tool for deriving our method and it is irrelevant what the actual reference frame is.}
Let us introduce
\begin{align}
\matW  =  \begin{bmatrix}  \matT_{ij} \end{bmatrix} \label{constructW} = &  \begin{bmatrix}
      \matT_{11} & \cdots & \matT_{1k} \\
      \vdots & \ddots & \vdots \\
      \matT_{k1} & \cdots & \matT_{kk}
    \end{bmatrix}  \\
 = &  \begin{bmatrix}
      \matT_{1\star}\matT_{\star1} & \cdots & \matT_{1\star }\matT_{\star k} \\
      \vdots & \ddots & \vdots \\
      \matT_{k \star }\matT_{\star 1} & \cdots & \matT_{k \star }\matT_{\star k}
    \end{bmatrix}  \\
 = &  \begin{bmatrix}
      \matT_{1\star}\matT_{1\star}^{-1} & \cdots & \matT_{1\star }\matT_{k\star}^{-1} \\
      \vdots & \ddots & \vdots \\
      \matT_{k \star }\matT_{1\star}^{-1} & \cdots & \matT_{k \star }\matT_{k\star}^{-1}
    \end{bmatrix} \\
 = & ~ \matU_1 \matU_2  \,,
\end{align}
where
$$
\matU_1  =   
\begin{bmatrix} 
\matT_{1\star} \\ 
\matT_{2\star} \\ 
\vdots \\ 
\matT_{k\star}
\end{bmatrix} \quad \text{and} \quad \matU_2 =  \begin{bmatrix}  \matT_{1\star}^{-1}, \matT_{2\star}^{-1}, \ldots, \matT_{k\star}^{-1} \end{bmatrix}\,.
$$
Using this notation, finding either $\matU_1$ or $\matU_2$ (up to an invertible linear transformation) gives the transitively consistent transformations. 

\begin{mydef}
   \add{Let $\matA \in \R^{p \times q}$. The set}
   $$\im(\matA) = \{\matA \vecx  ~\vert~  \vecx \in \R^{q}\}$$
   \add{is the column space of $\matA$ and the set} 
   $$\ker(\matA) = \{\vecx \in \R^{q} ~\vert~ \matA \vecx = \zerovec\}$$ 
   \add{is the null space of $\matA$.}
 \end{mydef} 

\add{Note that due to the invertability of $\matT_{i\star}$ (for all ${i=1,\ldots,k}$) it holds that the matrix $\matU_1$ has rank $d$ and thus the dimensionality of the column space $\im(\matU_1)$ of $\matU_1$ is exactly $d$.}

\begin{prop}\label{prop1}
\add{Let $\matZ = \matW - k \matI$.} The linear subspace $\ker(\matZ)$ has dimension $d$ and is equal to $\im(\matU_1)$. %
\end{prop}

\begin{proof}
First it is shown that the columns of $\matU_1$ are contained in the null space of $\matZ$, \ie $\im(\matU_1) \subseteq \ker(\matZ)$, and then it is shown that the null space of $\matZ$ has exactly dimension $d$.

Note that $\matU_2 \matU_1 = k\matI$, which we will make use of shortly. Multiplication of $\matU_1$ on the right to $\matW = \matU_1 \matU_2$ gives
\begin{eqnarray}
  & & \matW \matU_1 = \matU_1 \matU_2 \matU_1 = \matU_1 k\matI\\
  & \equivalent & \matW \matU_1 = k \matU_1 \\
  & \equivalent & \matW \matU_1 - k \matU_1  = \zerovec \\
  & \equivalent & (\matW - k \matI) \matU_1  = \zerovec \\
  & \equivalent & \matZ \matU_1  = \zerovec \quad \mbox{with } \matZ = \matW - k \matI \, . \label{nullspaceeqn}
\end{eqnarray} 
From (\ref{nullspaceeqn}) it can be seen that all the columns of $\matU_1$ are contained in the null space of $\matZ$, so $\im(\matU_1) \subseteq \ker(\matZ)$.

However, it still remains to be shown that the dimensionality of $\ker(\matZ)$ is exactly $d$, \ie $\im(\matU_1)$ spans the entire null space of $\matZ$ and not just a part of it. This is done by showing that there are no non-zero vectors $\vecx$ 
that are not contained in $\im(\matU_1)$ but are contained in $\ker(\matZ)$. 

Formally this is expressed by the requirement that the set
$$
A = \{ \vecx \in \R^{kd} ~\vert~ \vecx \neq \zerovec, \vecx \notin \im(\matU_1), \vecx \in \ker(\matZ)  \} 
$$
is empty.
Suppose now that $A$ is not empty, so it contains the element $\vecx \in \R^{kd}$. \add{Using the orthogonal decomposition theorem}, the vector $\vecx$ can be rewritten as $\vecx = \vecx_{\ker} + \vecx_{\im}$, where $\vecx_{\ker} \in \ker(\matU_1^T)$ and $\vecx_{\im} \in \im(\matU_1)$. The definition of $A$ states that $\vecx \notin \im(\matU_1)$, which implies that $\vecx_{\ker} \neq \zerovec$. Further, the definition of $A$ states that
\begin{align}
 & \vecx \in \ker(\matZ) \\
 \equivalent & \matZ \vecx = \zerovec \\
 \equivalent & \matZ (\vecx_{\ker} + \vecx_{\im}) = \zerovec \\
 \equivalent & \matZ \vecx_{\ker} + \matZ \vecx_{\im} = \zerovec \, .
 \intertext{Per definition $\vecx_{\im} \in \im(\matU_1)\subseteq \ker(\matZ)$, so it follows that $\matZ \vecx_{\im} = \zerovec$, leading to}
& \matZ \vecx_{\ker}  = \zerovec \, .\\
  \intertext{Multiplication of $\vecx_{\ker}^T$ on the left gives}
& \vecx_{\ker}^T \matZ \vecx_{\ker}  = 0\\
\equivalent & \vecx_{\ker}^T (\matU_1 \matU_2 - k\matI) \vecx_{\ker}  = 0 \\
\equivalent & \vecx_{\ker}^T \matU_1 \matU_2 \vecx_{\ker} - k \vecx_{\ker}^T  \vecx_{\ker}  = 0\\
\equivalent & \vecx_{\ker}^T \matU_2^T \matU_1^T \vecx_{\ker} - k \vecx_{\ker}^T  \vecx_{\ker}  = 0 \, .
\intertext{Per definition $\vecx_{\ker} \in \ker(\matU_1^T)$, so $\matU_1^T \vecx_{\ker} = \zerovec$, leading to }
& \vecx_{\ker}^T  \vecx_{\ker}  = 0 \\
\equivalent & \vecx_{\ker}  = \zerovec \, . \label{contr1}
\end{align}
Equation \eqref{contr1} is a contradiction to $\vecx_{\ker} \neq \zerovec$, thus, the set $A$ is empty.
\end{proof}

Proposition \ref{prop1} states that $\matU_1$ in \eqref{nullspaceeqn} can, up to an invertible linear transformation, be retrieved by finding the $d$-dimensional null space of $\matZ$.
Let $\matZ = \matU \matSigma \matV^T$ be the singular value decomposition (SVD) of $\matZ$. The $d$ columns of $\matV$ corresponding to the zero singular values span $\ker(\matZ)$ and give a solution to (\ref{nullspaceeqn}).

As we are only able to retrieve the transformations $\matT_{i \star}$ in the blocks of $\matU_1$ up to invertible linear transformations, w.l.o.g. we create a new version of $\matU_1$, call it $\matU_1'$, with the first $d \times d$ block being equal to the identity, as
\begin{align} \label{normalisedU1}
   \matU_1' = \matU_1 \matT_{1\star}^{-1} = 
   \begin{bmatrix} \matT_{1\star} \matT_{1\star}^{-1} \\ \matT_{2\star} \matT_{1\star}^{-1} \\ \vdots \\ \matT_{k\star} \matT_{1\star}^{-1} \end{bmatrix} \, . %
\end{align}

\subsection{Noisy Pairwise Transformations}
Up until this point, 
the matrix $\matU_1$ is obtained under perfect information, \ie the transitivity condition in Definition~\ref{transitivity} holds for all $\matT_{ij}$ transformations contained in the blocks of $\matW$. However, we are interested in the case when the transitivity condition does not hold due to measurement noise. Assume now that we have a noisy observation of $\matW$, denoted as $\tilde\matW$. Also, let the noisy version of $\matZ$ be $\tilde\matZ = \tilde \matW -k \matI$. Now, in general it is not the case that the null space of $\tilde \matZ$ is $d$-dimensional. Instead, the least-squares approximation of the $d$-dimensional null space is considered, which leads to the following optimisation problem:

\begin{problem}{Least-squares Transformation Synchronisation}\label{globallyConsTransLsq3}
\begin{eqnarray}
\nonumber
   & \underset{\hat\matT_{1\star},\ldots,\hat\matT_{k\star}}{\mbox{minimise}} \quad & \| \tilde\matZ \hat\matU_1 \|^2_F \\ 
   \nonumber
    & \mbox{subject to} \quad &  \vecu_i^T \vecu_j = 0 \quad \text{for all }~ i \neq j \\ %
    \nonumber
    & \quad & \| \vecu_i \| = 1 \quad \text{ for all } ~ i \\ %
    \nonumber
    & \quad & \hat\matU_1 = \begin{bmatrix} \vecu_1, \ldots, \vecu_d \end{bmatrix} \in \R^{kd \times d} \, .
\end{eqnarray}
\end{problem}
The rank-$d$ approximation of the null space of $\tilde \matZ$ can be retrieved using the SVD of $\tilde\matZ = \matU \matSigma \matV^T$. In this case the columns of $\matV$ corresponding to the $d$ smallest singular values span the rank-$d$ approximation of the null space of $\tilde\matZ$, giving $\hat\matU_1$, the estimate for $\matU_1$. By using \eqref{normalisedU1}, $\hat \matU_1'$ can be retrieved from $\hat \matU_1$.

\subsection{Affine Transformations in Homogeneous Coordinates}
In this section it is shown that the method is also applicable for invertible affine transformations, rather than invertible linear transformations. This is done by representing the $d$-dimensional affine transformations $\matT_{ij}$ by using $(d{+}1){\times}(d{+}1)$ homogeneous matrices.

\noindent Each affine transformation $\matT_{ij}$ can be written as
\begin{align} \label{affineBreakDown}
   \matT_{ij} = 
   \begin{bmatrix}
      \matA_{ij} & \zerovec \\
      \vect_{ij} & 1
   \end{bmatrix} \, , %
\end{align}
where $\matA_{ij}$ is the (invertible) linear $d \times d$ transformation matrix and $\vect_{ij}$ is the $d$-dimensional row vector representing the translation. The inverse of $\matT_{ij}$ is given by
\begin{align}
   \matT_{ij}^{-1} = 
   \begin{bmatrix}
      \matA_{ij}^{-1} & \zerovec \\
      - \vect_{ij} \matA_{ij}^{-1} & 1
   \end{bmatrix} \, . %
\end{align}
Similar to the linear case described in \eqref{constructW}, the matrix $\matW$ is constructed from all $\matT_{ij}$.
It is assumed that the matrix $\tilde\matW$, corresponding to the noisy observation of $\matW$, contains blocks that are proper affine transformations, \ie the last column of each block is $\begin{bmatrix} \zerovec & 1 \end{bmatrix}^T$.

A simple way to ensure that the synchronised transformations are affine transformations in homogeneous coordinates is to add the row vector $\vecz = \begin{bmatrix} z & z & \ldots & z \end{bmatrix} \in \R^{k(d{+1})}$, with $z = \begin{bmatrix} 0 & 0 & \ldots & 0 & 1\end{bmatrix} \in \R^{d{+}1}$, to the matrix $\tilde\matZ$. By adding the vector $\vecz$ to $\tilde\matZ$, the vector $\vecz^T$ is removed from the null space of $\tilde\matZ$. Using this approach, a solution is then found by solving Problem (\ref{globallyConsTransLsq3}) with the updated $\tilde\matZ$. Then the resulting $\hat\matU'_1$ gives an estimate of the first $d$ columns of $\hat\matT_{i\star} ~ (i=1,\ldots,k)$ and these are the columns we seek.

\subsection{Similarity Transformations}
Similarity transformations are transformations that allow for translations, isotropic scaling, rotations \add{and reflections}. To retrieve similarity transformations, the estimates of the synchronised affine transformations $\hat\matT_{i\star} ~ (i=1,\ldots,k)$ are determined first. The translation component $\hat\vect_{i\star}$ of $\hat\matT_{i\star}$ can directly be extracted from $\hat\matT_{i\star}$ since it has the structure presented in (\ref{affineBreakDown}). To obtain the scaling factor and the \add{orthogonal transformation}, the linear component $\hat\matA_{i\star}$ is factorised using SVD, resulting in $\hat\matA_{i\star} = \matU_{i\star} \matSigma_{i\star} \matV_{i\star}^T$. The orthogonal component $\hat\matQ_{i\star}$ is then given by 
\begin{align} \label{rotComp}
   \hat\matQ_{i\star} = \matU_{i\star} \matV_{i\star}^T \, ,
\end{align} 
 and the isotropic scaling factor $\hat s_{i\star}$ is given by 
 \begin{align} \label{geomMeanScaling}
   \hat s_{i\star} = \left( \prod_{j=1}^{d}{ \vert {(\sigma_{i\star})}_{jj} \vert }\right)^{\frac{1}{d}} \, ,
\end{align}
where ${(\sigma_{i\star})}_{jj}$ is the $j$-th element on the diagonal of $\matSigma_{i\star}$.

\begin{remark}
  \add{It can be shown that retrieving the orthogonal component as presented in eq.} \eqref{rotComp} \add{is the least-squares solution to the projection onto the set of orthogonal matrices. However, in eq.} \eqref{geomMeanScaling} \add{the isotropic scaling factor is retrieved as the geometric mean of the individual axis-aligned scaling factors. The least-squares solution to the projection onto the set of similarity transformations is given by the arithmetic mean,} \ie
$\hat s_{i\star}^{\text{lsq}} = \frac{1}{d} \sum_{j=1}^{d}{ \vert {(\sigma_{i\star})}_{jj} \vert }$.

\end{remark}

\subsection{Euclidean Transformations}
Similarity transformations without isotropic scaling are called euclidean transformations. To obtain euclidean transformations, the similarity transformations are extracted and the scaling factors $\hat s_{i\star}$ (for all $i = 1,\ldots,k$) are set to $1$.

\subsection{Rigid Transformations}
Euclidean transformations without reflections are called rigid transformations. Rigid transformations can be obtained by ensuring that the determinant of the rotational component $\hat\matQ_{i\star}$ described in (\ref{rotComp}) equals $1$. This can be achieved by setting 
\begin{align}
  & \hat\matQ_{i\star}  =  \matU_{i\star} \matD_{i\star} \matV_{i\star}^T , \mbox{ with}\\
  & \matD_{i\star} = \diag(1,\ldots,1, \det(\matV_{i\star}^T \matU_{i\star})) 
  \, .
\end{align}

\section{Experiments}
By generating ground truth data and adding Gaussian noise to it, we first compare the error of the synchronised transformations using our method to the error of the unsynchronised transformations.
Furthermore, the transformation synchronisation method is applied for solving the Generalised Procrustes Problem with missing points and with wrong correspondence assignments.

\subsection{Noisy Transformations}
In this section it is described how the ground truth transformations are generated, how noisy versions thereof are generated and eventually results of the transformation synchronisation method are presented.

\subsubsection{Ground Truth Transformations}\label{npp}
For the analysis of the performance of our method we generate a set of random transformations $\mathcal{T_{\star}} = \{\matT_{i\star}\}_{i=1}^{k}$, that are used in turn to generate the transitively consistent set of pairwise transformations $\mathcal{T} = \{\matT_{ij} = \matT_{i\star} \matT_{j\star}^{-1}\}_{i=1,j=1}^{k}$, serving as ground truth for the evaluation. The generation of $\mathcal{T_{\star}}$ is described in the following.

The dot-notation is used to illustrate that $\dot x$ is a random variable with a particular probability distribution. For generating the set $\mathcal{T_{\star}}$, we assume that the point-clouds that lead to the transformations have some structural similarity, \ie the transformations are not entirely random. In particular, the scaling factors, the translation components and the linear part of the transformation are restricted in the sense that they cannot be arbitrary. However, arbitrary orientations in $d$-dimensional space are allowed for.

The set $\mathcal{T_{\star}}$ contains the elements $\matT_{i\star} ~(i = 1, \ldots, k)$, which are samples of
\begin{align} \label{rndTrans}
    \dot\matT = 
    \begin{bmatrix}
      \dot s \dot\matQ \dot\matN & \zerovec \\
      \dot\vect & 1
   \end{bmatrix} \, , %
\end{align}
where ${\dot s \sim \unif(0.5,1.5)}$ is a scaling factor and ${\dot \vect \sim \unif(-2.5,2.5)^{d}}$ is a translation, with $\unif(a,b)^d$ denoting the $d$-dimensional uniform distribution having the open interval $(a,b)^d$ as support. 
Samples of the $d{\times}d$ random rotation matrix $\dot \matQ$ are drawn by extracting the rotational component of a non-singular random matrix as described in (\ref{rotComp}).
The $d{\times}d$ random noise matrix $\dot\matN$ is given by $\dot\matN= \matI + \mathbf{\dot\epsilon}$, where $\mathbf{\dot\epsilon} \sim \mathcal{N}(0,0.1^2)^{d\times d}$ is a $d{\times}d$ random matrix with each element having univariate normal distribution $\mathcal{N}(0,0.1^2)$. The purpose of creating the noise in the way using the random matrix $\dot\matN$ is to restrict the linear component in the transformation and thus to avoid ill-conditionedness with very high probability.

Depending on the type of transformation that is evaluated, the parameters of $\mathcal{T_{\star}}$ have different properties, which are summarised in Table \ref{tab:rndProps}.

\begin{table}[htbp]
\footnotesize
  \centering
  
    \begin{tabular}{lllll}
    \toprule
          & $\dot\matQ$ & $\dot\vect$ & $\dot s$ & $\dot\matN$ \\
    \midrule
    linear & $\vert \det \vert = 1$ & $= \zerovec$ & $\sim \unif(0.5,1.5)$ & $\sim \matI + \mathbf{\epsilon}$ \\
    affine & $\vert \det \vert = 1$ & $\sim \unif(0.5,1.5)^d$ & $\sim \unif(0.5,1.5)$ & $\sim \matI + \mathbf{\epsilon}$ \\
    similarity & $\vert \det \vert = 1$ & $\sim \unif(0.5,\texttt{}1.5)^d$ & $\sim \unif(0.5,1.5)$ & $= \matI$ \\
    euclidean & $\vert \det \vert = 1$ & $\sim \unif(0.5,1.5)^d$ & =1     & $=\matI$ \\
    rigid & $\det = 1$ & $\sim \unif(0.5,1.5)^d$ & =1     & $=\matI$ \\
    \bottomrule
    \end{tabular}%
    \caption{Properties of components of random transformations for different types of transformations generated according to (\ref{rndTrans}).} 
  \label{tab:rndProps}%
\end{table}%

Once the ground truth set $\mathcal{T}$ of transitively consistent transformations has been established, a noisy version thereof is synthetically created, as described in the next section.

The error $e(\mathcal{T}^1,\mathcal{T}^2)$ between two sets of pairwise transformations $\mathcal{T}^1 = \{ \matT_{ij}^1\}_{i,j=1}^k$ and $\mathcal{T}^2 = \{ \matT_{ij}^2\}_{i,j=1}^k$ is defined as
\begin{align} \label{avgError}
    e(\mathcal{T}^1,\mathcal{T}^2) = \frac{1}{k^2} \sum_{i,j=1}^{k} { \|\matT_{ij}^1 - \matT_{ij}^2 \|_F } \, .
\end{align}

\subsubsection{Additive Gaussian Noise}\label{addGN}
\newcommand{\nDraws}{100}
\begin{figure*}[!t!h] 
     \centerline{
        \subfigure{\includegraphics{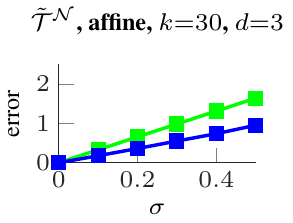}} \hfil
        \subfigure{\includegraphics{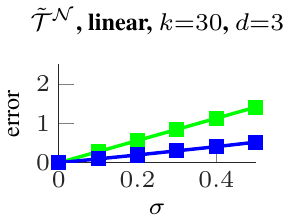}} \hfil
        \subfigure{\includegraphics{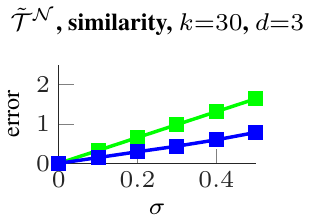} }\hfil
        \subfigure{\includegraphics{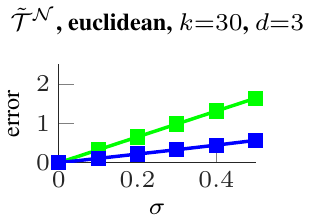} }\hfil
        \subfigure{\includegraphics{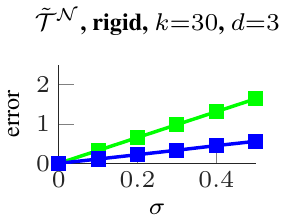} }
     }
     \centerline{
        \subfigure{\includegraphics{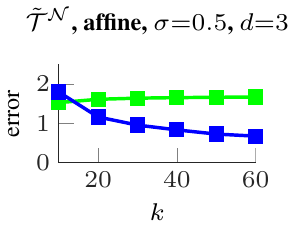} }\hfil
        \subfigure{\includegraphics{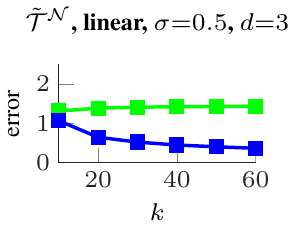}} \hfil
        \subfigure{\includegraphics{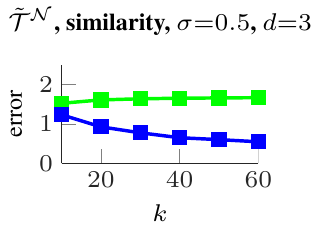} }\hfil
        \subfigure{\includegraphics{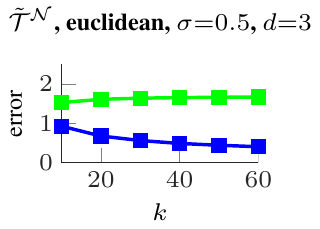}} \hfil
        \subfigure{\includegraphics{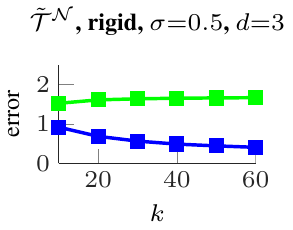}} 
     }
     \centerline{
        \subfigure{\includegraphics{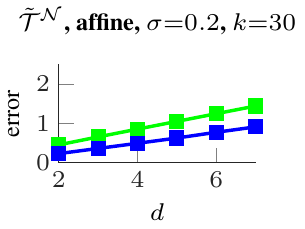}} \hfil
        \subfigure{\includegraphics{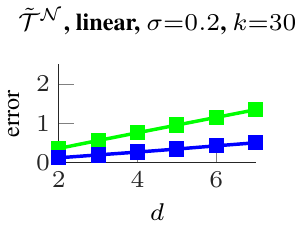}} \hfil
        \subfigure{\includegraphics{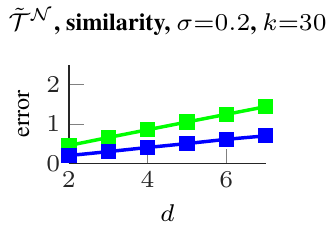}} \hfil
        \subfigure{\includegraphics{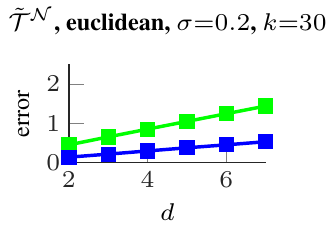}} \hfil
        \subfigure{\includegraphics{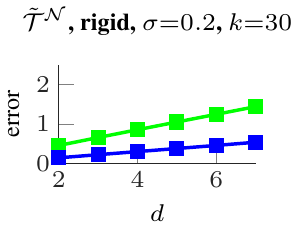} }
     }
    \caption{Error for additive normal noise $\mathcal{\tilde T}^{\mathcal{N}}$ for different configurations as specified in the graph title with one varying parameter (horizontal axis). Each row of graphs shows a particular varying parameter ($\sigma$, $k$ and $d$ from top to bottom) and each column of graphs shows a particular transformation type (affine, linear, similarity, euclidean, rigid, from left to right). The error as defined in \eqref{avgError} of the unsynchronised noisy transformations is shown in green and of the synchronised transformations in blue. Shown is the average error of $\nDraws$ randomly generated sets of ground truth transformations, where for each ground truth transformation $20$ runs of adding noise have been performed, resulting in a total of $2000$ simulations per graph.}
    \label{noisyTransResultsAddNorm}
\end{figure*}  

\add{The set of noisy pairwise transformations $\mathcal{\tilde T}^{\mathcal{N}}$ is created by adding to each element of the matrix $\matT_{ij}$ a sample from $\mathcal{N}(0,\sigma^2)$, which is conducted for all matrices $\matT_{ij} \in \mathcal{T}$ with $i \neq j$.} In the case of homogeneous transformation matrices $\matT_{ij}$, no noise is added to the last column, which shall always be $(\zerovec ~ 1)^T$.

Results of the simulations are shown in Fig.~\ref{noisyTransResultsAddNorm}. The first row of graphs show that for all types of transformations the error of synchronised transformations is smaller than the error of the unsynchronised transformations and that the slope of the error in the synchronised case is smaller than in the unsynchronised case. In the second row it can be seen that, even with a high amount of noise ($\sigma=0.5$), the error of the synchronised transformations decreases with an increasing number of objects $k$. As anticipated, with increasing $k$ there is more information available, directly resulting in a lower error. The last row of graphs shows that increasing the dimensionality results in an increasing error; however, the error of the synchronised transformations increases slower than for the unsynchronised ones.

\subsection{Generalised Procrustes Analysis}
\newcommand{\nDrawsA}{500}
\begin{figure*}[!ht]
      \centerline{
        \subfigure{\includegraphics{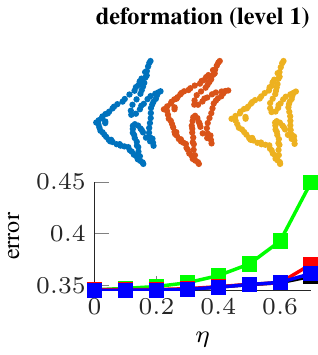}}  \hfil
        \subfigure{\includegraphics{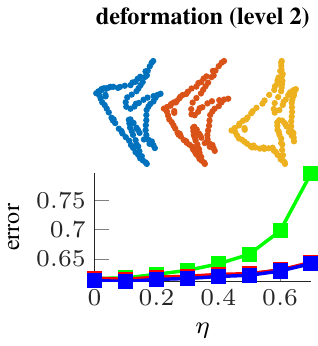}}  \hfil
        \subfigure{\includegraphics{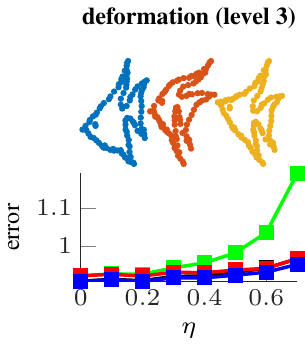}}  \hfil
        \subfigure{\includegraphics{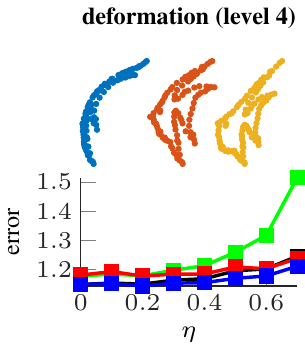}}  \hfil
        \subfigure{\includegraphics{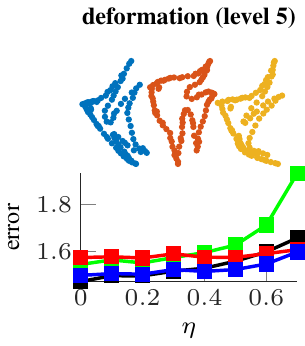}} \hfil
        }
      \centerline{
        \subfigure{\includegraphics{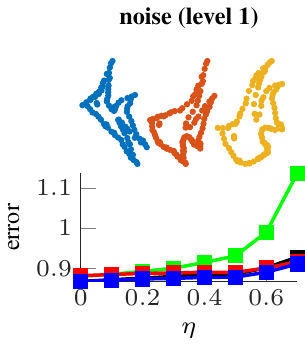}}  \hfil
        \subfigure{\includegraphics{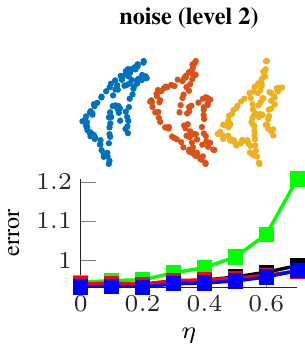}} \hfil
        \subfigure{\includegraphics{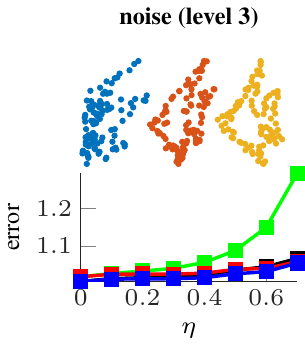}}  \hfil
        \subfigure{\includegraphics{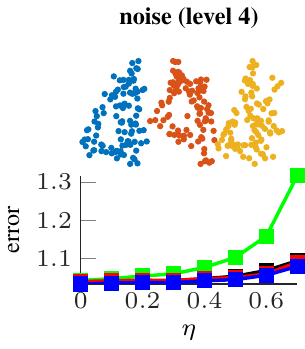}}  \hfil
        \subfigure{\includegraphics{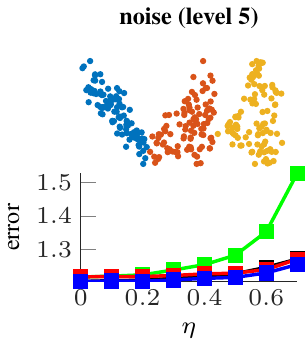}} \hfil 
      }
        \caption{Average shape error for the reference-based (green), iterative mean shape-based (black), synchronisation-based (blue) and stratified (red) method for solving the GPP with missing data. The horizontal axis shows the probability $\eta$ that a point is considered missing. At the top of each graph three shapes according to the particular level of deformation or noise are depicted. 
        Shown is the average shape error for $\nDrawsA$ draws of missing data in each graph. 
        In every run $k=30$ out of $K=100$ shapes are randomly selected, where each shape comprises $n=98$ points in $d=2$ dimensions.}
    \label{gppSynced}
\end{figure*}

In addition to evaluating the synchronisation of noisy pairwise transformations we have applied our method for solving 
the Generalised Procrustes Problem (GPP), which is done on the one hand with missing data and on the other hand with wrong correspondence assignments. For both simulations the 2D fish shapes from the Chui-Rangarajan data set \cite{Chui:2003wr} with different levels of deformation and noise have been used (refer \cite{Chui:2003wr} for more details). For each level of deformation and noise the data set contains $K=100$ shapes, each comprising $n=98$ points in $d=2$ dimensions. 

Finding the similarity transformation that best aligns two shapes, which is a subroutine for the evaluated reference-based, \add{the iterative mean shape-based} and the synchronisation-based method, is performed by an AOP implementation with symmetric scaling factors \cite{Horn:1988bq}. In the reference-based solution of GPP one shape is randomly selected as reference and all other shapes are aligned with the reference. \add{For the iterative mean shape-based method the initial reference shape is selected randomly and then the mean shape is iteratively updated.} In the synchronisation-based solution of GPP all $k^2$ pairwise AOPs are solved first, followed by the synchronisation of the resulting transformations in order to aggregate all information contained in the pairwise transformations. \add{Additionally, the stratified GPA method proposed in} \cite{Bartoli:2013vv} \add{is evaluated for solving the GPP. In our experiments we have observed that by using the stratified GPA method the linear part of the resulting transformations may collapse to the zero matrix; in order to enable a comparison with the other methods in these cases the linear part of the transformation has simply been set to the identity matrix.}

In the missing data experiments as well as the wrong correspondence experiments for each single run ${k=30}$ out of ${K=100}$ shapes are randomly selected.
For the experiments in the missing points case the missing points are simulated by discarding points according to a given probability.
For the experiments with wrong correspondences the correct correspondences are randomly disturbed in order to simulate wrong correspondences.

In contrast to \emph{solving} the AOPs, in both experiments the computation of the error is performed using the original shape (\ie with all points and with perfect correspondences). With that we investigate up to which amount recovering the original shapes from corrupt shape data is possible. The average shape error of a set of shapes $\mathcal{X} = \{\matX_i\}_{i=1}^k$ is defined as $e(\mathcal{X}) = \frac{1}{k^2} \sum_{i,j=1}^k \| \matX_i - \matX_j \|_F$.

\subsubsection{Missing Points}\label{gppMissingPoints}

In every run, additionally to randomly selecting $30$ out of $100$ shapes, each data point of a shape is considered to be missing with probability $\eta$. As \add{the implemented methods solve} the AOP only for common points in each pair of shapes, values for $\eta$ larger than $0.7$ have not been investigated because with $\eta>0.7$ the cases that the number of common points in a pair of shapes is less than $d=2$ occur too frequently (for $d$-dimensional data, there must be at least $d$ points in each shape in order to result in a system that is not under-determined). Also, for $\eta \le0.7$ it is possible that the number of common points in a pair of shapes is less than $d=2$; in these cases the draw of missing data is simply repeated.

In Fig.~\ref{gppSynced} the resulting error of the reference-based, the iterative mean shape-based, the synchronisation-based and the stratified solution of the GPP with missing data are shown for different levels of deformation and noise. It can be seen that even with an increasing amount of missing data, when using the synchronisation-based method the error increases only slightly, whilst the error of the reference-based method increases significantly with a larger amount of missing points. \add{With respect to the error, the transformation synchronisation method performs only marginally better than the iterative mean shape-based method and the stratified method. However, the average runtimes for solving a single GPP instance was $0.007$\,s for the reference-based method, $0.162$\,s for the synchronisation-based method, $1.932$\,s for the iterative mean shape-based method and $2.265$\,s for the stratified method, illustrating that our method performs significantly better than all other methods when taking runtime and error into account at the same time.}

\subsubsection{Wrong Correspondence Assignments}\label{wrongCorAss}

\newcommand{\nDrawsB}{500}

\begin{figure*}[!th] 
  \begin{minipage}{0.735\textwidth}
      \centerline{
        \subfigure{\includegraphics{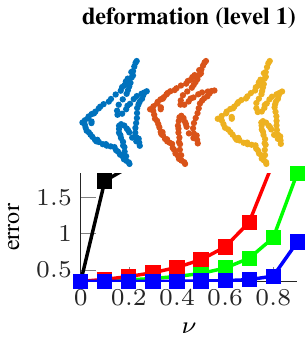}}  %
        \subfigure{\includegraphics{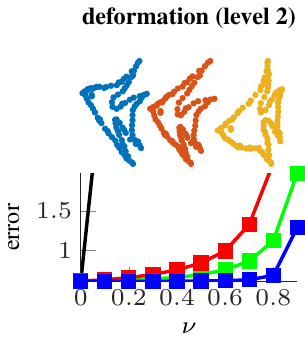}}  %
        \subfigure{\includegraphics{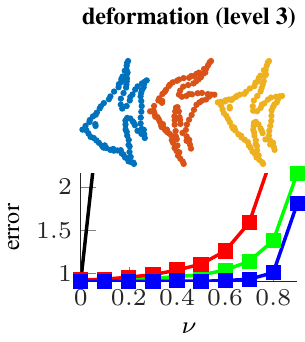}}  %
        \subfigure{\includegraphics{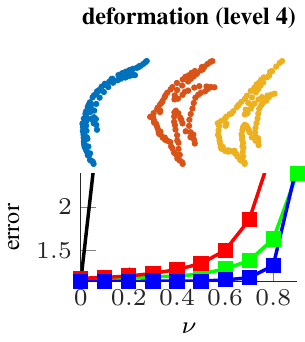}}  %
        }
         \vspace{6mm}
      \centerline{
        \subfigure{\includegraphics{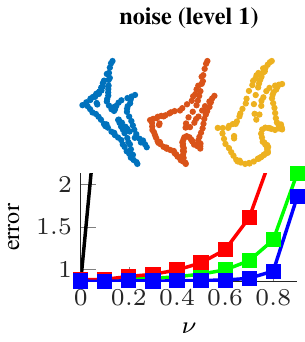}} %
        \subfigure{\includegraphics{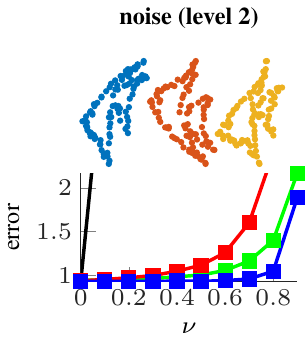}}  %
        \subfigure{\includegraphics{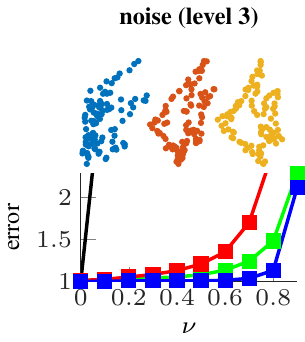}} %
        \subfigure{\includegraphics{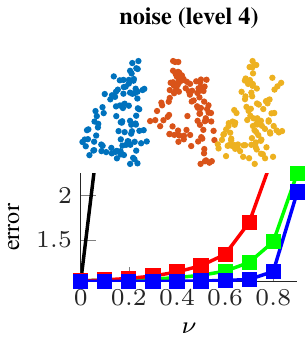}}  %
      }
    \end{minipage}
    \begin{minipage}{0.3\textwidth}
    \centerline{
      \subfigure{\includegraphics{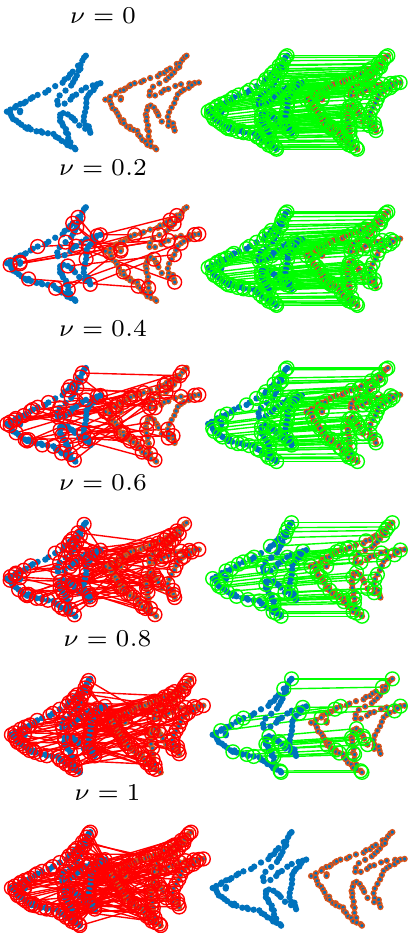} }
    }
    \end{minipage}
        \caption{Average shape error for the reference-based (green), iterative mean shape-based (black), synchronisation-based (blue) and stratified (red) method for solving the GPP with wrong correspondences. The horizontal axis shows the proportion $\nu$ of wrong correspondences. At the top of each graph three shapes according to the particular level of deformation or noise are depicted. 
        Shown is the average shape error for $\nDrawsB$ runs of disturbing correspondence assignments in each graph. In every run $k=30$ out of $K=100$ shapes are randomly selected, where each shape comprises $n=98$ points in $d=2$ dimensions. In the right-most column examples of the correspondence assignments between a pair of shapes are depicted for different values of $\nu$ in each row. In order to keep the visualisation as coherent as possible, the wrong correspondences (red lines) and the correct correspondences (green lines) are shown separately.}
    \label{gppWrongAssignments}
\end{figure*}

Additionally to the case of missing points, we have applied our method to solve the GPP with wrong correspondence assignments between shapes. In order to mimic practical applications, where it is frequently the case that the true correspondences are unknown and thus it must be assumed that wrong correspondences are present, we do not make any efforts to correct these wrong correspondences (such as using RANSAC \cite{Fischler:1981cv} or permutation synchronisation \cite{Pachauri:2013wx}). Instead, for each pair of shapes the AOP is solved whilst being aware that some of the points in the one shape have wrong counterparts in the other shape. Of course this will have influence on the resulting transformations. Thus, the objective of the simulations described in this section is to assess to what extent the transformations from shapes with wrong correspondences can be reconstructed using transformation synchronisation.

In every run, additionally to randomly selecting $30$ out of $100$ shapes, the correspondences between the $n$ points in each shape are disturbed. For disturbing the correspondence assignments each pair of shapes that is to be aligned is considered independently. For that, a proportion of ${\nu \in [0,1]}$ points from the total number of $n$ points is selected. Then, as correspondences between the pair of point-clouds $\matX_i, \matX_j \in \R^{n \times 2}$ are implicitly given by the ordering of the rows, the rows corresponding to the previously selected points are reordered randomly in one of the point-clouds, directly resulting in disturbed correspondence assignments between the pair of point-clouds $\matX_i, \matX_j$. 

In Fig.~\ref{gppWrongAssignments} the reference-based, the iterative mean shape-based, the synchronisation-based and the stratified solution of the GPP with wrong correspondences are shown for different levels of deformation and noise. On the right of Fig.~\ref{gppWrongAssignments} examples of the correspondences between pairs of shapes are depicted for different values of $\nu$. 
 
It can be seen that for different levels of deformation and different levels of noise with $70\%-80\%$ of wrong correspondences the outcome is only marginally affected when using our proposed method. \add{In contrast, all other evaluated methods result in significantly larger errors, which can be explained by the fact that our method is the only one that is able to make use of the information that is contained in all pairwise transformations.}

\section{Conclusion}
The alignment of multiple (corresponding) point-clouds simultaneously is generally tackled by iteratively aligning all point-clouds to some reference. Whereas this approach is biased (selecting a fixed reference) or initialisation-dependent (using the adaptive mean as reference) we have presented a method that is completely unbiased and does not depend on initialisation.

Our key observation is that the underlying noise-free transformations can be retrieved from the null space of a matrix that can directly be obtained from pairwise alignments. Whilst related approaches for rotation matrices \cite{Hadani:2011tw,Hadani:2011hb,Singer:2011ba} or permutation matrices \cite{Pachauri:2013wx} have been proposed, we have generalised the synchronisation method to handle general linear and affine transformations as well as similarity, euclidean and rigid transformations. \add{Experimentally we were able to demonstrate that the proposed method is able to effectively reduce noise from the set of pairwise transformations and to solve the Generalised Procrustes Problem at least as good as existing approaches for the missing data case whilst significantly outperforming other methods for the presented wrong correspondence case.}

 %

%

%

\ifcvprfinal
\section*{Acknowledgements} 
Supported by the Fonds National de la Recherche, Luxembourg (5748689, 6538106, 8864515).
\fi

{\small
\bibliographystyle{ieee}
\bibliography{references}
}

\end{document}